\documentclass[letterpaper]{article}

\usepackage{aaai25}

\usepackage{times} 
\usepackage{helvet} 
\usepackage{courier} 
\usepackage[hyphens]{url} 
\usepackage{graphicx} 
\usepackage{graphicx} 
\usepackage{natbib} 
\usepackage{caption} 
\usepackage{subcaption}
\usepackage{soul}
\usepackage{url}
\usepackage[utf8]{inputenc}
\usepackage{amsmath}
\usepackage{booktabs}
\usepackage{amsthm}
\usepackage[ruled, vlined,linesnumbered]{algorithm2e}
\usepackage[switch]{lineno}
\usepackage{xcolor}
\usepackage{amssymb}
\usepackage{multicol}
\usepackage{multirow}
\usepackage{array}
\usepackage{tabularray}
\usepackage{rotating}
\usepackage{tikz}
\usepackage{enumitem}
\usepackage{paralist}

\newtheorem{theorem}{Theorem}
\newtheorem{lemma}{Lemma}

\newtheorem{definition}{Definition}

\newcommand{\cmdp}{\mathcal{C}}

\newcommand{\states}{\mathcal{S}}

\newcommand{\act}{\mathcal{A}}
\newcommand{\trans}{\delta}

\newcommand{\reward}{r}

\newcommand{\distr}{\mathcal{D}}
\newcommand{\mixdistr}{mix\_distr}
\newcommand{\detdistr}{det\_distr}

\newcommand{\reals}{\mathbb{R}}

\newcommand{\discount}{\gamma}
\newcommand{\expect}{\mathbb{E}}

\newcommand{\hist}{h}

\newcommand{\E}{\mathbb{E}}

\newcommand{\thr}{\Delta}
\newcommand{\Rset}{\mathbb{R}}
\newcommand{\Nset}{\mathbb{N}}

\newcommand{\probm}{\mathbb{P}}

\newcommand{\eps}{\varepsilon}

\newcommand{\acost}{Cost}

\newcommand{\rew}{\reward}
\newcommand{\node}{n}

\newcommand{\payoff}{\mathit{Payoff}}

\newcommand{\explconst}{C}

\newcommand{\hor}{T}
\newcommand{\pen}{c}
\newcommand{\sinit}{s_0}

\newcommand{\tstep}{i}
\newcommand{\paretoest}{\mathcal{P}}
\newcommand{\paretotrue}{{P}}

\newcommand{\FUNCTION}[2]{\STATE \textbf{Function} #1(#2) \begin{ALC@g} }
\newcommand{\ENDFUNCTION}{\end{ALC@g} \STATE \textbf{EndFunction}}

\newcommand{\uct}{\mathit{expl}}
\newcommand{\conv}{\mathit{conv}}
\newcommand{\prune}{\mathit{prune}}
\newcommand{\thract}{\thr_{\mathit{act}}}

\newcommand{\imcost}{\bar{c}}

\DeclareMathOperator*{\argmax}{arg\,max}
\DeclareMathOperator*{\argmin}{arg\,min}

\title{Threshold UCT: Cost-Constrained Monte Carlo Tree Search with Pareto Curves}

\author {
Martin Kurečka\textsuperscript{\rm 1},
Václav Nevyhoštěný\textsuperscript{\rm 1},
Petr Novotný\textsuperscript{\rm 1},
Vít Unčovský\textsuperscript{\rm 1}\\
}
\affiliations {
\textsuperscript{\rm 1}Faculty of Informatics, Masaryk University\\Botanická 68a, 612 00 Brno, Czech Republic\\
petr.novotny@fi.muni.cz
}
\begin{document}

\maketitle

\begin{abstract}
Constrained Markov decision processes (CMDPs), in which the agent optimizes expected payoffs while keeping the expected cost below a given threshold, are the leading framework for safe sequential decision making under stochastic uncertainty. Among algorithms for planning and learning in CMDPs, methods based on Monte Carlo tree search (MCTS) have particular importance due to their efficiency and extendibility to more complex frameworks (such as partially observable settings and games). However, current MCTS-based methods for CMDPs either struggle with finding safe (i.e., constraint-satisfying) policies, or are too conservative and do not find valuable policies.
We introduce Threshold UCT (T\nobreakdash-UCT), an online MCTS-based algorithm for CMDP planning. Unlike previous MCTS-based CMDP planners, T\nobreakdash-UCT explicitly estimates Pareto curves of cost-utility trade-offs throughout the search tree, using these together with a novel action selection and threshold update rules to seek safe and valuable policies. Our experiments demonstrate that our approach significantly outperforms state-of-the-art methods from the literature.
\end{abstract}

\section{Introduction}

\paragraph{Safe Decision Making and MCTS} Monte-Carlo tree search (MCTS) has emerged as the de-facto method for solving large sequential decision making problems under uncertainty~\cite{MCTSsurvey}. It combines the scalability of sampling-based methods with the robustness of heuristic tree search, the latter feature making it easily extendable to settings with partial observability~\cite{SV:2010:POMCP}, multiple agents~\cite{SilEtAl:18:alphazero-science}, or settings with history-dependent optimal decisions~\cite{ChaENR:2018:RAMCP}. While MCTS-based methods demonstrated remarkable efficiency in optimizing the agent's performance across diverse domains, the deployment of autonomous agents in real-world domains necessitates balancing the agent performance with the \emph{safety} of their behavior. In AI planning and reinforcement learning, the standard way of modeling safety issues is via the \emph{constrained decision-making} framework. Here, apart from the usual reward signals, the agents are also collecting \emph{penalties} (or \emph{costs}), and the objective is to maximize the expected accumulated reward under the constraint that the expected accumulated cost is below a given threshold \( \thr \). Compared with ad-hoc reward shaping, the constrained approach provides explicit and domain-independent way of controlling agent safety. Hence, safe and efficient probabilistic planning (and indeed, also model-free reinforcement learning, where algorithms such as MuZero~\cite{muzero} are built on top of efficient MCTS planners) necessitates the development of stable and sample-efficient cost-constrained MCTS algorithms.

\paragraph{Key Components of Constrained MCTS} An efficient constrained MCTS-based algorithm must be able to identify \emph{safe} and \emph{valuable} policies. 

Finding safe policies (i.e., those that do not exceed the cost threshold) requires identifying ``dangerous'' (in terms of future cost) decisions and keeping track of \emph{cost risk} accumulated in stochastic decisions: a \( 50/50 \) gamble which incurs cost \( C  \) if lost contributes at least \( C/2 \) towards the expected cost of the policy irrespective of the gamble's outcome.

An agent which never moves might be safe but never does anything useful. To identify reward-valuable policies among the safe ones, the algorithm must not be constrained beyond the requirements given by the threshold \( \thr \) and hence it must be able to reason about the \emph{trade-off} between rewards and costs during both tree search and actual action selection. 

\paragraph{Limitations of previous approaches} Two prominent examples of constrained MCTS-based algorithms are CC-POMCP~\cite{ccpomcp} and RAMCP~\cite{ChaENR:2018:RAMCP}. While these algorithms represented significant steps towards practical constrained decision making, they exhibit fundamental limitations in identifying both safe and valuable policies. 

\emph{Safety limitations:} A usual way of tracking the cost risk is updating the current threshold \( \thr \) appropriately after each decision. As we discuss in Section~\ref{sec:tuct}, both CC-POMCP and RAMCP perform this update in an unsound manner and might thus produce policies violating the cost constraint even if a safe policy exists within the explored part of the tree. 

\emph{Value limitations:} Both CC-POMCP and RAMCP compute randomized policies, which are necessary for optimality in constrained decision-making scenarios~\cite{Altman:1999:CMDP-book}. However, their reasoning about the reward-cost payoff is incomplete. RAMCP does not use the cost information during the tree search phase (which mimics the standard UCT~\cite{KS06}) at all: costs are only considered in the actual action selection phase, when a linear program (LP) encoding the constraints is solved over the constructed tree. Although the LP drives the agent to satisfy the constraints, the data used to construct the LP are sampled using a cost-agnostic search procedure which might lead to sub-optimal decisions. CC-POMCP, on the other hand, is a Lagrangian dual method in which the Lagrangian multiplier \( \lambda \) represents a concrete reward-cost tradeoff to be used in both tree search and action selection. The key limitation of the Lagrangian approach is the scalar nature of the tradeoff estimate \( \lambda \): unless \( \lambda \) quickly converges to the optimal tradeoff, the algorithm will collect data according to either overly conservative or overly risky policies, yielding instability that hampers convergence to a valuable policy. This behavior was witnessed in multiple of our experiments (Section~\ref{sec:experiments}).

\paragraph{Our contributions}  We introduce \emph{Threshold UCT (T\nobreakdash-UCT)}, an online MCTS-based constrained decision-making algorithm designed so as to seek policies that are both safe and valuable. T-UCT achieves this by estimating the \emph{Pareto curves} of the cost-payoff tradeoff in an online fashion. In every step, T-UCT uses these Pareto estimates to play a randomized action mixture optimal w.r.t. the current cost threshold. This is done both during the actual agent's action selection and during tree search. The latter phase resolves the exploration/exploitation tradeoff through a variant of the UCT (\emph{upper confidence bound on trees}~\cite{KS06}) approach, adapted to the constrained setting. In particular, T-UCT's exploration is cost-sensitive and the algorithm comes with a new threshold update rule, which ensures that the agent is not incorrectly driven into excessive risk. 
%
%
We evaluate T-UCT on benchmarks from the literature, including a model of an autonomous vehicle navigating the streets of Manhattan. Our experiments show that T-UCT significantly outperforms state-of-the-art algorithms, demonstrating a notable sample efficiency and stable results across different environments.

\paragraph{Further related work}

The problem of constrained decision making under uncertainty, formalized via the notion of constrained Markov decision processes~\cite{Altman:1999:CMDP-book} has received lot of attention in recent years, with approaches based on linear programming~\cite{Altman:1999:CMDP-book,PouMPKGB15:constrained-POMDP,cmdp-bayesian}, heuristic search~\cite{UndH10:constrained-pomdp-online,STW16:BWC-POMDP-state-safety}, primal-dual optimization~\cite{ChoGJP:2017:risk-constrained-policy,DBLP:conf/iclr/TesslerMM19}, local policy search~\cite{pmlr-v70-achiam17a}, backward value functions~\cite{pmlr-v119-satija20a}, or Lyapunov functions~\cite{NEURIPS2018_4fe51490}. Unlike these works, our paper focuses on the MCTS-based approach to the problem, due to its scalability and extendibility to more complex domains.

The RAMCP algorithm has been extended into an AlphaZero-like MCTS-based learning algorithm RAlph in~\cite{BraCNV:20:RL-risk-averse}. In this paper, we do not consider function approximators and instead focus on the correctness and efficiency of the underlying tree search method. The recent learning algorithm ConstrainedZero~\cite{moss2024constrainedzero} computes deterministic policies, considers chance constraints, and does not track the cost risk of stochastic decision, thus solving a problem different from ours.

Our work is also related to \emph{multi-objective (MO)} planning~\cite{BarettNarayanan:CHVI,JMLR:v15:vanmoffaert14a}. There, the task is to estimate tradeoffs among multiple payoff functions w.r.t. various solution concepts (including Pareto optimality). T-UCTs approach of performing full Bellman updates of Pareto curves during backpropagation is similar in spirit to  \emph{convex hull MCTS}~\cite{convex-hull-mcts}. However, MO approaches do not consider constrained optimization and thresholds; the main novelty of T-UCT is using the Pareto curves to guide the tree search towards valuable constraint-satisfying parts via threshold-based action selection and sound threshold updates.

Constrained decision making is also related to \emph{risk-sensitive} planning and learning (e.g., \cite{nips-cvar,NIPS2015_64223ccf,PraF:22:risk-rl-book,HYV16:risk-pomdps,Vulcan,Kretisnky-cvar}), where safety is enforced by putting a constraint on some \emph{risk-measure} of the underlying policy. Some risk measures, such as \emph{chance constraints,} can be expressed in our framework by encoding accumulated payoffs into states.








\section{Preliminaries}

\label{sec:prelims}

We denote by \( \distr(X) \) the set of all probability distributions over a finite support \( X \). We formalize the constrained decision making problem via the standard notion of constrained Markov decision processes (CMDPs).

\begin{definition}
A \emph{constrained Markov decision process (CMDP)} is a tuple \( \cmdp = (\states, \act,\trans, \rew, \pen, \sinit ) \) where:
\begin{compactitem}
\item \( \states \) is a finite set of \emph{states,}
\item \( \act \) is a finite set of \emph{actions,}
\item \( \trans \colon \states \times \act \rightarrow \distr(\states) \) is a probabilistic \emph{transition function;} we abbreviate \( \trans(s,a)(t) \) to \( \trans(t\mid s,a) \),
\item \( \rew \colon \states \times \act \times \states \rightarrow \Rset\) is a \emph{reward function,}
\item \( \pen \colon \states \times \act \times \states \rightarrow \Rset \) is a \emph{cost function,} and
\item \( \sinit \in \states\) is the initial state 

\end{compactitem}
\end{definition}

\paragraph{CMDP dynamics}

CMDPs evolve identically to standard MDPs. A \emph{history} is an element of \( (\states\act)^*\states \), i.e., a finite alternating sequence of states and actions starting and ending in a state. A \emph{policy} is a function assigning to each history a distribution over actions. 

Under a given policy \( \pi \), a CMDP evolves as follows: we start in the initial state; i.e., the initial history is \( \hist_0 = \sinit \). Then, for every timestep \( \tstep \in \{0,1,2,\ldots\} \), given the current history \( \hist_\tstep = \sinit a_0 s_1 a_1 \ldots s_{\tstep-1}a_{\tstep-1}s_\tstep \), the next action \( a_\tstep \) is sampled from \( \pi \):  \( a_\tstep \sim \pi(\hist_\tstep) \). The next state \( s_{\tstep+1} \) is sampled according to the transition function, i.e, \( s_{\tstep+1} \sim \trans(s_\tstep,a_\tstep) \). Then, the agent obtains the reward \( r(s_\tstep,a_\tstep,s_{\tstep+1}) \) and incurs the cost \( \pen(s_\tstep,a_\tstep,s_{\tstep+1}) \). The current history is updated to \( \hist_{\tstep+1} = \sinit a_0 s_1 a_1 \ldots s_{\tstep-1}a_{\tstep-1}s_\tstep a_{\tstep} s_{\tstep+1}  \) and the process continues in the same fashion \emph{ad infinitum.}

We denote by $\probm^\pi_{}(E)$ the probability of an event $ E $ under policy $\pi$, and by $ \E^\pi_{}[X] $ the expected value of a random variable $ X $ under $\pi$. We denote by \( |h| \) the \emph{length} of history \( h \), putting \( |\sinit a_0 \ldots s_{\tstep-1}a_{\tstep-1}s_\tstep| = \tstep \).

We will sometimes abuse notation and use a history in a context where a state is expected - in such a case, the notation refers to the last state of a history. E.g. \( \trans(-\mid h,a) \) denotes a transition probability distribution from the last state of \( h \) under action \( a \). 

\paragraph{Problem statement}

Under a fixed policy \( \pi \), the agent accumulates (with possible discounting) both the rewards and costs over a finite \emph{decision horizon} \( \hor \):
\begin{align*}
\payoff_\pi &= \E^\pi[\sum_{\tstep = 0}^{\hor-1} \gamma_\rew^\tstep \cdot \rew(s_{\tstep}, a_{\tstep},s_{\tstep+1}) ], \\
\acost_\pi &= \E^\pi[\sum_{\tstep = 0}^{\hor-1} \gamma_\pen^\tstep \cdot \pen(s_{\tstep}, a_{\tstep},s_{\tstep+1}) ],
\end{align*}
where \( \gamma_\rew,\gamma_\pen \in (0,1] \) are reward and cost \emph{discount factors.} 

Our goal is to maximize the accumulated payoff while keeping the accumulated cost below a given threshold. Formally, given a CMDP, the horizon \( \hor \in \Nset\), discount factors \( \gamma_\rew,\gamma_\pen \in (0,1] \), and a \emph{cost threshold } \( \thr \in \Rset_{\geq 0} \), our task is to solve the following constrained optimization problem:
\begin{align*}
\max_{\pi} ~&\payoff_\pi \\
\text{subject to } &\acost_\pi \leq \thr.
\end{align*}
Our algorithm tackles the above problem in an online fashion, producing a local approximation of the optimal constrained policy.
%




\section{Threshold UCT} \label{sec:tuct}

We propose a new algorithm for CMDP planning, \emph{Threshold UCT (T-UCT)}.
Like many other MCTS-based algorithms, T-UCT only requires access to a generative simulator of the underlying CMDP, i.e., an algorithm allowing for an efficient sampling from \( \trans(-| s,a) \), given \( (s,a) \); and providing \( \rew(s,a,s') \) and \( \pen(s,a,s') \) for given \( (s,a,s') \).


\paragraph{History-action values, feasibility}

We consider payoffs achievable by a policy \( \pi \) after witnessing a history \( h \) and possibly also playing action \( a \):
\begin{align*}
\payoff_\pi(h) &= \E^\pi[\sum_{\tstep = |h|}^{\hor-1} \gamma_\rew^{\tstep-|h|} \cdot \rew(s_\tstep, a_\tstep, s_{\tstep+1}) \mid h], \\
\payoff_\pi(h, a) &= \E^\pi[\sum_{\tstep = |h|}^{\hor-1} \gamma_\rew^{\tstep-|h|} \cdot \rew(s_\tstep, a_\tstep, s_{\tstep+1}) \mid h, a],
\end{align*}
where \( (\cdot|h) \) is a condition of producing history \( h \) in the first \( |h| \) steps and \( (\cdot|h,a)\) is a condition that \( a \) is played immediately after witnessing \( h \). The quantities \( \acost_\pi(h) \) and \( \acost_\pi(h,a) \) are defined analogously. We say that \( \pi \) is \( \thr \)-feasible from \( h \) if \( \acost_\pi(h) \leq \thr \).

\paragraph{Achievable vectors} A vector \( (c, r) \in \Rset \times \Rset \) is \emph{achievable} from history \( h \) if there exists a policy \( \pi \) such that \( \acost_\pi(h) \leq c \) and \( \payoff_\pi(h) \geq r \). Similarly, we say that \( (c,r) \) is achievable from \( (h,a) \) if \( \acost_\pi(h, a) \leq c \) and \( \payoff_\pi(h, a) \geq r \) for some \( \pi. \)

We write \( (c',r') \preceq (c,r) \) if \( c' \geq c \) and \( r' \leq r \). A \( \preceq \)\nobreakdash-closure of a set \( X \subseteq \Rset \times \Rset \) is the set of all vectors \( (c',r') \in \Rset \times \Rset \) s.t. \( (c',r') \preceq (c,r) \) for some \( (c,r) \in X \).

\paragraph{Pareto sets}
A \emph{Pareto set} of history \( h \) 
is the set of all vectors achievable from \( h \), while the Pareto set of \( (h,a) \) is the set of all vectors achievable from \( (h,a) \).

It is known~\cite{MO-MDP-verif,BarettNarayanan:CHVI} that the Pareto sets are (i) convex (since we allow randomized policies), and (ii) \( \preceq \)-closed (i.e., if \( (c,r) \) belongs to the set and \( (c',r') \preceq (c, r) \), then \( (c',r') \) also belongs to the set). From (ii) it follows that a Pareto set is wholly determined by its \emph{Pareto curve}, i.e. the set of all points maximal w.r.t. the \( \preceq \)-ordering. Furthermore, in finite MDPs, the Pareto curve is piecewise linear, with finitely many pieces. The whole Pareto set can then be represented by a finite set of \emph{vertices,} i.e., points in which the piecewise-linear curve changes its slope; indeed, the Pareto set is the \( \preceq \)\nobreakdash-closure of the convex hull of the set of vertices.
%
In what follows, we denote by \( \paretotrue(h) \) and \( \paretotrue(h,a) \) these finite representations of the Pareto sets of \( h \) and \( (h,a) \), respectively.

\paragraph{Bellman equations for Pareto sets} Pareto sets in CMDPs obey local optimality equations akin to classical unconstrained MDPs. To formalize these, we need additional notation. The sum \( X + Y \) is the standard Minkowski sum of the sets of vectors. For a vector \( (a,b) \) we define \( X\cdot (a,b) = \{(x\cdot a, y \cdot b) \mid (x,y) \in X\} \), with  \( X \cdot a \) as a shorthand for \( X \cdot (a,a) \). It is known~\cite{BarettNarayanan:CHVI,mo-stochgames} that for the finite-vertex representation of Pareto sets it holds: 
\begin{align}
\paretotrue(h) &= \prune\Big(\bigcup_{a \in \act}\paretotrue(h,a) \Big)  \label{eq:bellman-h} \\
\paretotrue(h,a) &= \prune \Big( \sum_{t \in \states} \trans(t|h,a)\big(\paretotrue(hat) \cdot(\gamma_\pen,\gamma_\rew)\nonumber \\ &\qquad+\{(\pen(h,a,t),\rew(h,a,t))\}\big)\Big), \label{eq:bellman-ha}
\end{align}
%
where the \( \prune \) operator removes all points that are \( \preceq \)\nobreakdash-dominated by a convex combination of some other points in the respective set.

\SetKwFunction{puct}{ThresholdUCT}
\SetKwProg{proc}{Procedure}{}{}
\SetKwProg{func}{Function}{}{}
\SetKwFunction{getleaf}{GetLeaf}
\SetKwFunction{expand}{Expand}
\SetKwFunction{rollout}{Rollout}
\SetKwFunction{propagate}{Propagate}
\SetKwFunction{getact}{GetActionDist}
\SetKwFunction{updpareto}{UpdateThr}

\begin{algorithm}[]
\caption{Threshold UCT}

    
    \label{alg:pareto-uct}
 \proc{\( \puct(\hor, \thr) \)}{
 \( h \gets s_0  \)\;
 \While{\( \hor > 0 \)}{
 \Repeat{timeout}{
 {\color{blue} $leaf \gets \getleaf(h, \thr)$}\;
                 $\mathit{newleaf} \gets \expand(leaf)$\label{aline:expansion}\;
 \propagate($\mathit{newleaf}, h$);
 }
             \label{aline:timeout}
             
 {\color{blue}
         $\sigma \gets  \getact(h, \thr, 0)$       \label{aline:main-act-select}\;
 }
         \( a \sim \sigma \);\,\, play \( a \) and observe new state \( t \) \label{aline:actual-transition}\;
 {\color{blue}
         $\thr \gets \updpareto(\thr, \sigma, a, t)$\label{alg:pareto-uct:update-1}\;
 }
         $h \gets hat$;
         \( \hor \gets \hor - 1 \)\;
 }}

 \func{$\getleaf(h, \tilde{\thr})$}{
 \While{$h$ is not a leaf}{
 {\color{blue} $\sigma \gets \getact(h, \tilde{\thr}, 1)$\;}
            \( a \sim \sigma \)$;\,\, t \sim \hat{\delta}(-\mid h, a)$\;
 {\color{blue} $\tilde{\thr} \gets \updpareto(\tilde{\thr}, \sigma, a, t)$\label{alg:pareto-uct:update-2}}\;
            $h \gets hat$\;
 }
 \Return $h$
}

 \proc{$\propagate(h,\mathit{root})$}{
 {\color{blue}
    
    \( (c,r) \leftarrow \rollout(h)\)\label{aline:rollout}\;
       $\paretoest(h) \gets \{(c,r), (0,0)\}$\label{aline:rollout-end}\;
 }
 \While{$h \neq \mathit{root}$}{
        $\text{write } h \text{ as } h'as$;
        $h \gets \text{$h'$}$\;
 {\color{blue}
 update \( \paretoest(h,a) \) via equation~\eqref{eq:bellman-ha}\;
 update \( \paretoest(h) \) via equation~\eqref{eq:bellman-h}\;
        
 }
 }
 }
    
 \func{\(\getact(h,\thr,e)\)}{
       $\widetilde{\paretoest} \gets \prune\big(\displaystyle\bigcup_{a'\in \act} \paretoest(h, a') + \{\uct(h, a') \cdot(-e, e)\}\big)$\label{aline:prune-and-uct}\;
       $\widetilde{C}^- \gets \{c \mid (c,r) \in \widetilde{\paretoest} \wedge c \leq \thr\}$\;$\widetilde{C}^+ \gets \{c \mid (c,r) \in \widetilde{\paretoest} \wedge c \geq \thr\}$\;
 \If{$\widetilde{C}^-\! = \emptyset$\label{aline:allbad-start}}{
              \label{alg:pareto-uct-getact:ret1}
       $a \gets \argmin_{a'} \min \{c \mid (c,r) \in \widetilde{\paretoest}(h, a')\}$\;
 \Return $\detdistr(a)$\label{aline:allbad-end}\;
 }
 \ElseIf{$\widetilde{C}^+\! = \emptyset$ \label{aline:allgood-start}}{
       $a \gets \argmax_{a'} \max \{r \mid (c,r) \in \widetilde{\paretoest}(h, a')\}$\;
 \Return $\detdistr(a)$\label{alg:pareto-uct-getact:ret2}\label{aline:allgood-end}\;}
 \Else{
       \label{aline:mixing-start}
       $c_l \gets \max \widetilde{C}^-$\;
       $a_l \gets$ action realising $c_l$\;
       $c_h \gets \min \widetilde{C}^+$\;
       $a_h \gets$ action realising $c_h$\label{alg:pareto-uct-getact:ch}\;
       $\sigma_h \gets \frac{\thr - c_l}{c_h - c_l}$;\,\,
       $\sigma_l \gets 1 - \sigma_h$\;
 \Return $\mixdistr(\sigma_l, a_l, \sigma_h, a_h)$\label{aline:mixing-end}\;
 }
 }
\end{algorithm}


\paragraph{T-UCT: Overall structure} T-UCT is presented in Algorithm~\ref{alg:pareto-uct}. 
%
%
It follows the standard Monte Carlo tree search (MCTS) framework, with blue lines highlighting parts that conceptually differ from the setting with unconstrained payoff optimization (the whole procedure \texttt{GetActionDist} is constraint-specific, and hence we omit its coloring). The algorithm iteratively builds a \emph{search tree} whose nodes represent histories of the CMDP, with child nodes of \( h \) representing one-step extensions of \( h \). Each node stores additional information, in particular the estimates of \( \paretotrue(h) \) and \( \paretotrue(h,a) \), denoted by \( \paretoest(h) \) and \( \paretoest(h,a) \) in the pseudocode. The tree structure is global to the whole algorithm and not explicitly pictured in the pseudocode. 

The algorithm uses transition probabilities \( \trans \) of the CMDP.
If these are not available (e.g., when using a generative model of the CMDP),
we replace \( \trans \) with a sample estimate based on the visit count of transitions during the tree search.
The estimates are  updated globally with every sample from the simulator (omitted in the pseudocode). 
In what follows, we denote by $\hat{\trans}$ either the real transition probabilities or, if these are unavailable, their sample estimates.

Initially, the tree contains a single node representing the history \( h_0 = \sinit \).
In each decision step \( 1,2,\ldots,T \), T-UCT iterates, from the current root \( h \), the standard four stages of MCTS until the expiry of a given timeout. The stages are: (i) \emph{search,} where the tree is traversed from top to bottom using information from previous iterations to navigate towards the most promising leaf (function \( \texttt{GetLeaf} \)); (ii) \emph{leaf expansion,} where a successor of the selected leaf is added to the search tree (line~\ref{aline:expansion}); followed by (iii) \emph{rollout:} where the Pareto curve of the new leaf is estimated, e.g., via Monte Carlo simulation following some default policy (lines~\ref{aline:rollout}--\ref{aline:rollout-end}), or possibly via a pre-trained predictor. Note that we also add a tuple $(0,0)$ to the curve to make the exploration ``cost-optimistic'', even if the rollout policy is unable to find safe paths. Finally, T-UCT performs (iv) \emph{backpropagation,} of the newly obtained information (particularly, the Pareto curve estimates) from the leaf back to the root (the rest of procedure \texttt{Propagate}). We provide more details on the individual stages below.

After this, T-UCT computes and outputs an action distribution $\sigma$ from which action \( a \) to be played is sampled (lines~\ref{aline:main-act-select}--\ref{aline:actual-transition}). The action is performed, and a new state \( t \) of the environment (\ref{aline:actual-transition}) and the immediate reward and cost (omitted from pseudocode) are incurred. The cost threshold \( \thr \) is then updated (line~\ref{alg:pareto-uct:update-1}, details below) to account for the cost incurred, discounting, and cost risk of the stochastic decision. The node \( hat \) becomes the new root of the tree and the process is repeated until a decision horizon is reached.

The key components distinguishing T-UCT from standard UCT are the backpropagation, action selection, and threshold update. In the following, we present these components in more detail.

\paragraph{Backpropagation} T-UCT's backpropagation is similar to \emph{convex-hull MCTS}~\cite{convex-hull-mcts}: performing full Bellman updates of the Pareto set estimates according to equations~\eqref{eq:bellman-h} and~\eqref{eq:bellman-ha} (using estimates \( \hat{\delta} \) in lieu of \( \delta \) when necessary). 

\paragraph{Action selection} The function \texttt{GetActionDist} computes an action distribution based on the current estimates of Pareto curves. In the search phase of the algorithm (where \( e \) is set to 1),
we encourage playing under-explored actions with an exploration bonus as depicted on line~\ref{aline:prune-and-uct}:
$$\uct(h, a) = C \cdot \alpha(h) \cdot \sqrt{\frac{\log{N(h)}}{N(h, a)+1}},$$
where $N(h)$ and $N(h, a)$ are the visit counts of node $h$ and action $a$ in $h$, respectively,
$C$ is a fixed static exploration constant, and $\alpha(h)$ is a dynamic adjustment of the exploration constant
equal to the difference between the maximum and minimum achievable value (payoff or cost) in $h$.
Note that for each \( a \in \act \), the bonus is applied to all vertices in \( \paretoest(h,a) \); thereafter the exploration-augmented estimate \( \widetilde{\paretoest} \) of \( \paretotrue(h) \) is computed via~\eqref{eq:bellman-h}.
When playing the actual action, the exploration bonus is disabled by setting \( e=0 \).

If there is no feasible policy based on \( \widetilde{\paretoest} \) (i.e., policy satisfying \( \acost_\pi(h) \leq \thr\)), we return a distribution which deterministically selects the action minimizing the future expected cost (lines~\ref{aline:allbad-start}--\ref{aline:allbad-end}).
Conversely, if the expected future cost can be kept below \( \thr \) no matter which action is played, we deterministically select the action with the highest future \( \thr \)-constrained payoff (lines~\ref{aline:allgood-start}--\ref{aline:allgood-end}).
 
Otherwise, we find vertices \( v_l = (c_l,r_l) \) and \( v_h = (c_h,r_h) \) of \( \widetilde{\paretoest} \) whose future cost estimates are the nearest to \( \thr \) from below and above, respectively. Due to pruning, necessarily \( r_l \leq r_h \). We identify actions \( a_l \) and \( a_h \) that realize the cost-reward tradeoff represented by these vertices, and mix them in such a proportion that the resulting convex combination of \( c_l \) and \( c_h \) equals \( \thr \), so that the ``cost budget'' is maxed out (lines~\ref{aline:mixing-start}--\ref{aline:mixing-end}). The resulting mixture is returned to the agent or the tree search procedure.



\paragraph{Threshold update} After an action \( a \) is played and its outcome \( t \) observed, the cost threshold \( \thr \) must be updated to account for (a) the immediate cost incurred, (b) the discount factor $\discount_\pen$, and (c) the predicted contribution of outcomes other than \( t \) to the overall \emph{expected} cost achieved by the policy from \( h \). The update is performed by the \texttt{UpdateThr} function, described below.

If the transition $hat$ has not been expanded yet, the update involves only the subtraction of the immediate cost and the division by $\discount_\pen$.
Otherwise \texttt{UpdateThr}(\( \thr,\sigma,a,t \)) first computes the intermediate threshold \( \thract \): If \( \sigma \) is deterministic, we set \( \thract \) to \( \thr \); if \( \sigma \) is a stochastic mix of two different actions, we check if the action \( a \) sampled from \( \sigma \) is the \( a_l \) or \( a_h \) from line~\ref{aline:mixing-end};  accordingly, we set \( \thract \) to either \( c_l \) or \( c_h \).

Based on \( \thract \), the value of \( \thr \) is updated to a new value \(\thr'\) in a way depending on ``how much'' $\thract$ is feasible from $ha$.
In the following, we use \( \conv  \) to denote the convex hull operator. There are three cases to consider, similar to those in \texttt{GetActionDist}:

\emph{Case ``mixing''}: In the first case, there exists a maximal \( \rho \in \Rset\) such that \( (\thract,\rho) \in \conv(\paretoest(h,a)) \); i.e., by~\eqref{eq:bellman-ha} $\rho$ is the maximum real satisfying
\begin{equation}
\label{eq:conv-comb}
\begin{aligned}
(\thract,\rho) = &\sum_{s\in\states} \hat{\trans}(s\mid h,a)\cdot \big[(c_{s}\cdot \gamma_\pen,r_{s}\cdot \gamma_\rew)\\&+(c(h,a,s),r(h,a,s)) \big],
\end{aligned}
\end{equation}
for some points \( (c_{s},r_{s}) \) such that \( (c_{s},r_{s}) \in \conv(\paretoest(has)) \) for each \( s \in \states \).
\texttt{UpdateThr}(\( \thr,\sigma,a,t \)) then updates \( \thr \) to
\begin{equation}
\label{eq:update-mixing}
\thr' \gets \pen_t.
\end{equation}
This ensures that no matter which \( t \) is sampled, the overall expected cost in \( ha \) is bounded by \( \thract \), provided that the points in the Pareto curve estimates are achievable. 

\emph{Case ``surplus''}: The second case is when there is a surplus in the cost budget, i.e., $\thract > c_{\max} = \max \left\{ c \mid (c, r) \in \paretoest(h,a)\right\}$.
A naive update would be to ignore the surplus and continue as if $\thract$ was $c_{\max}$.
Per the Pareto curve estimates, there would be no decrease in payoff since indeed, under the estimated dynamic $\hat{\trans}$, the optimal policy does not exceed the cost $c_{\max}$.
However, $\hat{\trans}$ can be incorrect and too cost-optimistic. Hence, ignoring the surplus \( \thract - c_{\max} \) could over-restrict the search in future steps.
Instead, we distribute the surplus over all possible outcomes of \( a \) proportionally to the outcomes' predicted cost. Formally, \texttt{UpdateThr}(\( \thr,\sigma,a,t \)) identifies a vertex \( (c_{\max},\rho) \in \paretoest(h,a)\) and computes vectors \( (c_{s},r_{s}) \) satisfying~\eqref{eq:conv-comb} with left-hand side set to \( (c_{\max},\rho) \). Then, it computes
\begin{equation}
\label{eq:update-surplus}
\thr' \gets c_t + (\thract - c_{\max})\frac{B - c_t}{\imcost(h, a) + \gamma_\pen B - c_{\max}},
\end{equation}
where $\imcost(h, a) = \sum_{s} \hat{\trans}(s \mid h, a)\cdot \pen(h, a, s)$ is the expected immediate cost,
and \( B = T \cdot \max_{(s,a,t)} \pen(s, a, t)\) is an upper bound on the accumulated cost of any trajectory.

\emph{Case ``unfeasible''}: Finally, if $\thract$ is unfeasible according to $\paretoest(h,a)$,  i.e., when $\thract < c_{\min} = \min \left\{ c \mid (c, r) \in \paretoest(h,a)\right\}$, \texttt{UpdateThr}(\( \thr,\sigma,a,t \)) distributes all the
missing cost to the current outcome \( t \) (unlike in the previous case). Formally, the update is the following:
\begin{equation}
\label{eq:ramcp-update}
\thr' \gets c_t - \frac{c_{\min} - \thract}{\hat{\trans}(t \mid h, a) \gamma_\pen},
\end{equation}
where the \( c_{s} \)-values are, again, computed by applying~\eqref{eq:conv-comb} to a vector \( (\pen_{\min}, \rho) \in \paretoest(h,a) \).

\paragraph{Theoretical analysis}

The threshold update function \texttt{UpdateThr} enjoys two notable properties that are important for the algorithm's correctness:
First a), the update never increases the estimated expected threshold, i.e.,
\begin{equation}
    \label{eq:thr-update-exp-decrease}
 \thr \geq \expect \left[\pen(h, a, t) + \discount_\pen \cdot \thr'\right]\footnote{The expectation is taken under the distribution $\hat{\delta}$.},
\end{equation}
and b), the updated value is feasible according to the estimates, i.e., 
\begin{equation}
\label{eq:thr-update-feasible}
        \thr' \geq c_{t},
\end{equation}
whenever $\thr$ is feasible according to the estimates.

The properties are important for two reasons. First, they ensure the asymptotical convergence of the algorithm, and second,
they prevent the algorithm from an excessive increase of the threshold under limited exploration.

Concerning the asymptotical convergence, we prove the $\varepsilon$\nobreakdash-soundness of T-UCT in the sense formalized by the following theorem, proved in  Supplementary material (\ref{app:proofs}).

\begin{theorem}
\label{thm:convergence}
Let $\mathcal{C}$ be a CMDP and $\thr$ a threshold such that
there exists a $\thr$-feasible policy.
Then for every $\varepsilon \in \reals^+$, there exists $n$ such that
T-UCT with $n$ MCTS iterations per action selection is $(\thr + \varepsilon)$-feasible.
\end{theorem}

We explore the importance of the properties \eqref{eq:thr-update-exp-decrease} and \eqref{eq:thr-update-feasible} in further sections.
Our experiments reveal that in certain cases, RAMCP violates the inequality \eqref{eq:thr-update-exp-decrease}. This leads to an excessive increase of the threshold and results in RAMCP ignoring the cost constraint.
In Supplementary material (\ref{app:ccpomcp-counterexample}), we further show that CC-POMCP can violate \eqref{eq:thr-update-feasible} by not taking the action outcomes into account during threshold updates, yielding a threshold-violating policy.


\section{Experiments} \label{sec:experiments}
\paragraph{Baseline Algorithms}
We compare T-UCT to two state-of-the-art methods for solving CMDPs: \textit{CC-POMCP}~\cite{ccpomcp} and \textit{RAMCP}~\cite{ChaENR:2018:RAMCP}.

\textit{CC-POMCP} is a dual method based on solving the Lagrangian relaxation of the CMDP objective:
\begin{align}
    \label{eq:lagrangian}
 \min_{\lambda \geq 0} \max_{\pi} \payoff_{\pi} + \lambda \cdot (\thr - \acost_{\pi}).
\end{align}
CC-POMCP performs stochastic gradient descent on $\lambda$ while
continuously evaluating \eqref{eq:lagrangian} via the standard UCT search.
For a fixed $\lambda$, the maximization of \eqref{eq:lagrangian} yields
a point on the cost-payoff Pareto frontier, where a larger value of $\lambda$ induces more
cost-averse behavior. A caveat of the method is its sample inefficiency when $\lambda$ converges slowly.


\textit{RAMCP}
is a primal method combining MCTS with linear programming.
The search phase of RAMCP greedily optimizes the payoff
in a standard single-valued UCT fashion, completely ignoring the constraint.
Consequently, the cost is considered only in the second part of the action selection phase,
where the algorithm solves a linear program to find a valid probability flow through
the sampled tree (which serves as a local approximation of the original CMDP)
such that the expected cost under the probability flow satisfies the constraint
and the payoff is maximized.

\paragraph{Benchmarks}

\newcommand{\todocite}{{\color{red}\cite{??}}}
\newcommand{\trp}{p_{trap}}
\newcommand{\slp}{p_{slide}}
\newcommand{\tps}{t}
\newcommand{\sat}{SAT}
\newcommand{\satm}{\sat_M}
\newcommand{\satw}{\sat_W}
\newcommand{\sats}{\sat_S}
\newcommand{\cnv}{SV}
\newcommand{\cnvm}{\cnv_M}
\newcommand{\cnvw}{\cnv_W}
\newcommand{\cnvs}{\cnv_S}
\newcommand{\mpo}{\hat{r}}
\newcommand{\mco}{\hat{c}}

\newcommand{\rndrmap}[1]{%
\begin{tikzpicture}[scale=0.3]%
 \foreach \line [count=\y] in #1 {%
 \foreach \pix [count=\x] in \line {%
 \draw[fill=pixel \pix] (\x,-\y) rectangle +(1,1);%
 }%
 }%
\end{tikzpicture}%
}

\definecolor{pixel .}{HTML}{d4d2cb}
\definecolor{pixel G}{HTML}{f5d742}
\definecolor{pixel W}{HTML}{000000}
\definecolor{pixel T}{HTML}{f5427e}
\definecolor{pixel B}{HTML}{509647}

\def\map{
 {W,W,W,W,W,W,W,W},
 {W,G,.,.,T,G,G,W},
 {W,T,.,B,.,.,T,W},
 {W,G,T,W,T,G,G,W},
 {W,W,W,W,W,W,W,W},
}
\begin{figure}
\centering

\begin{subfigure}[t]{0.22\textwidth}
    \centering
\rotatebox{0}{\resizebox{100px}{!}{
\begin{tikzpicture}[scale=0.3]%
 \foreach \line [count=\y] in \map {%
 \foreach \pix [count=\x] in \line {%
 \draw[fill=pixel \pix] (\x,-\y) rectangle +(1,1);%
 }%
 }%
 \draw[->, blue] (4.5,-2.5) -- (5.1,-2.5) -- (5.5,-1.5);
 \draw[->, blue] (4.5,-2.5) -- (5.1,-2.5) -- (5.5,-3.5);
 \draw[->, blue] (4.5,-2.5) -- (5.7,-2.5);
\end{tikzpicture}%
}}
\caption{A Gridworld map with the initial tile (green), golds (yellow), and traps (red). The possible outcomes of moving right are depicted by blue arrows.}
\label{fig:map}
\end{subfigure}
\hfill
\begin{subfigure}[t]{0.22\textwidth}
    \centering
\resizebox{100px}{!}{
\includegraphics[width=0.23\textwidth]{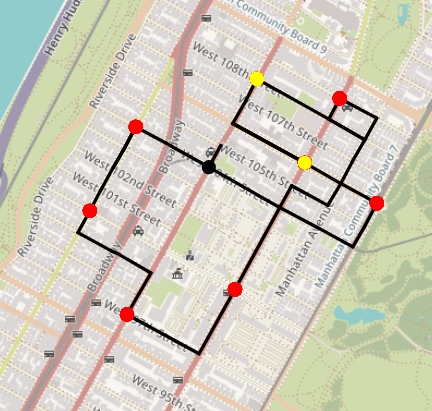}
}
\caption{Manhattan map depicting the initial junction and the agent's trajectory (black), maintenance points (red) and accepted requests (yellow).}
\label{fig:manhattan}
\end{subfigure}

\caption{The Gridworld and Manhattan environments.}

\end{figure}
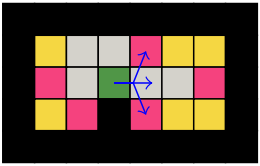
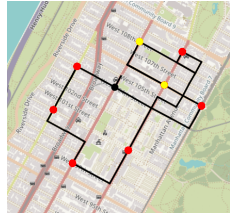

We evaluated the algorithms on three tasks: two of them are variants of the established \emph{Gridworld} benchmark~\cite{BraCNV:20:RL-risk-averse,UndH10:constrained-pomdp-online}; the third is based on the \emph{Manhattan} environment~\cite{manhattan-paper}, where the agent navigates through mid-town Manhattan, avoiding traffic jams captured by a stochastic model.

\paragraph{Task: Avoid} The setup (see Figure~\ref{fig:map}) involves navigating the agent (robot)
through a grid-based maze with four permissible actions: moving left, right, down, or up.
The agent's movements are subject to stochastic perturbations:
it can be displaced in a direction
perpendicular to its intended movement with probability $\slp$. The agent's objective is to visit special ``gold'' tiles to receive a reward
while avoiding ``trap'' tiles, which, with probability $\trp$, incur a cost of 1 and terminate the environment (the agent suffers fatal damage).


The task is evaluated on a set of 128 small maps {($6\times6$ grid, 5 gold, approx. $10^3$ states)} and 64 large maps {($25\times25$ grid, 50 gold, approx. $10^{17}$ states)}. All maps are sampled from our map generator, which guarantees a varying distribution of traps, walls, and gold.
We provide the generator and the resulting map datasets \emph{GridworldSmall} and \emph{GridworldLarge} in the Supplementary material.

\paragraph{Task: SoftAvoid}
The setup is similar to \emph{Avoid} (including the same generated maps) except that fixed amount of the cost $\trp$ is deterministically incurred upon
stepping on a trap and the environment does not terminate in that case. Thus the goal is to collect as much gold as possible
    while keeping the number of visited trap tiles low.

 
\paragraph{Task: Manhattan}
    The setup involves navigation in the eponymous \emph{Manhattan} environment implemented in \cite{manhattan-paper}.
The agent moves through mid-town Manhattan ({$250$ junctions, 8 maintenance points, approx. $10^{21}$ states}) while the time required to pass each street is sampled from a distribution estimated from real-life traffic data.
Periodically, selected points on the map request maintenance. If the request is accepted, the agent receives a reward of 1 when reaching the point of the request; however,
if it does not deliver the order in the given time limit, it incurs a cost of 0.1.

\begin{figure*}[ht!]
 \centering
 \includegraphics[width=\textwidth]{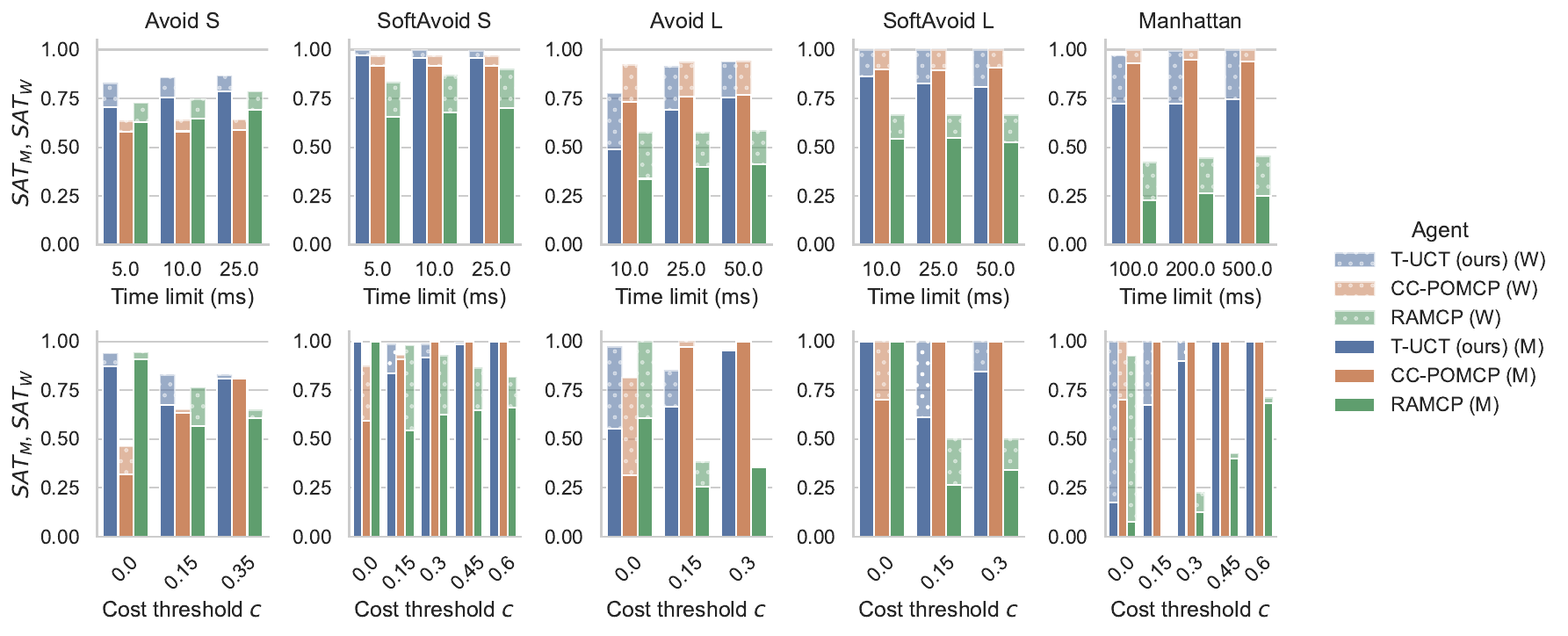}
 \caption{
The fractions of satisfied configurations in the \textit{mean} (plain) and the \textit{weak} (dotted) sense.
Upper row: The overall fraction of satisfied instances across varied time limits.
Lower row: Breakdown of the fractions of satisfied instances across varied thresholds with the time limit set to the maximum value ($25$, $50$, or $500$, respectively).
}
\label{fig:results}
\end{figure*}


\begin{table}[ht!]
    \centering
    \begin{tabular}{lcc}
        \toprule
        \textbf{Task} & \textbf{Reward r} & \textbf{Cost c} \\
        \midrule
        Avoid & 1  & 1 \\
        SoftAvoid & 1 & $\trp$ \\
        Manhattan & 1 & 0.1 \\
        \bottomrule
    \end{tabular}
    \caption{Reward and cost summary.
    In all environments, the agent receives rewards deterministically either on collecting gold or finishing the order.
    In \emph{SoftAvoid} and \emph{Manhattan}, the agent incurs the cost deterministically on the trigger (stepping on a trap or not delivering the order in time),
    while in \emph{Avoid} the cost is incurred with probability $\trp$.}
    \label{tab:task_reward_cost}
\end{table}


\paragraph{Implementation}
We implemented T-UCT, RAMCP, and CC-POMCP algorithms within a common C++ based MCTS framework,
so as to maximize the amount of shared code among the algorithms,
reducing the influence of implementation details on the experiments.
%
The \emph{Gridworld} environments are implemented as part of our C++ codebase, while the \emph{Manhattan} environment
is built on top of the Python implementation provided by~\cite{fimdp}. 

\paragraph{Experimental setup}
Each task was evaluated with various settings of parameters such as the risk threshold $\thr$, map, or
slide probability. We evaluated \emph{GridworldSmall} tasks (both \emph{Avoid} and \emph{Soft-Avoid}) 
on 1144, \emph{GridworldLarge} on 2304, and \emph{Manhattan} on 600 configurations. The full description of configurations and of the hardware setup is in the Supplementary material (\ref{app:experiment-configs}).

We performed 300 independent runs on each configuration.
The algorithms were given a fixed time per decision summarized in Figure \ref{fig:results}; note the time limit on \emph{Manhattan} is looser so as to compensate for
the slower Python implementation of the environment.
The runs on \emph{GridworldSmall}-based tasks had \( T = 100 \) steps, while the runs on \emph{GridworldLarge} and \emph{Manhattan} maps had \( T = 200 \) steps.



\paragraph{Metrics}
For each experiment configuration, we computed the mean collected payoff $\mpo$, mean collected cost $\mco$, and its standard deviation.
Based on these, we report two statistics detailed below: the fraction of satisfied instances $\sat$ and the mean payoff achieved on satisfied instances.

Since the empirical cost $\mco$ suffers from stochastic inaccuracy, we define two levels of satisfaction.
The \textit{mean satisfaction} metric $\satm$ is simply the fraction of instances where the empirical cost satisfies the constraint, i.e., $\mco \leq \thr$.
For the \textit{weak satisfaction} $\satw$ we weaken the constraint which allows us to give it a statistical significance; it is the fraction of instances where the t-test ($\alpha=0.05$) rejects the hypothesis that the real expected cost of the algorithm is more than $\thr + 0.05$.
Since it is meaningless to compare the payoffs of the algorithms on instances where they violate the constraint,
for each baseline algorithm, we further identify the instances where the algorithm and T-UCT both satisfy the constraint (in the weak sense)
and compute the average payoff on these instances.

\begin{figure*}[ht!]
    \centering
    \includegraphics[width=\textwidth]{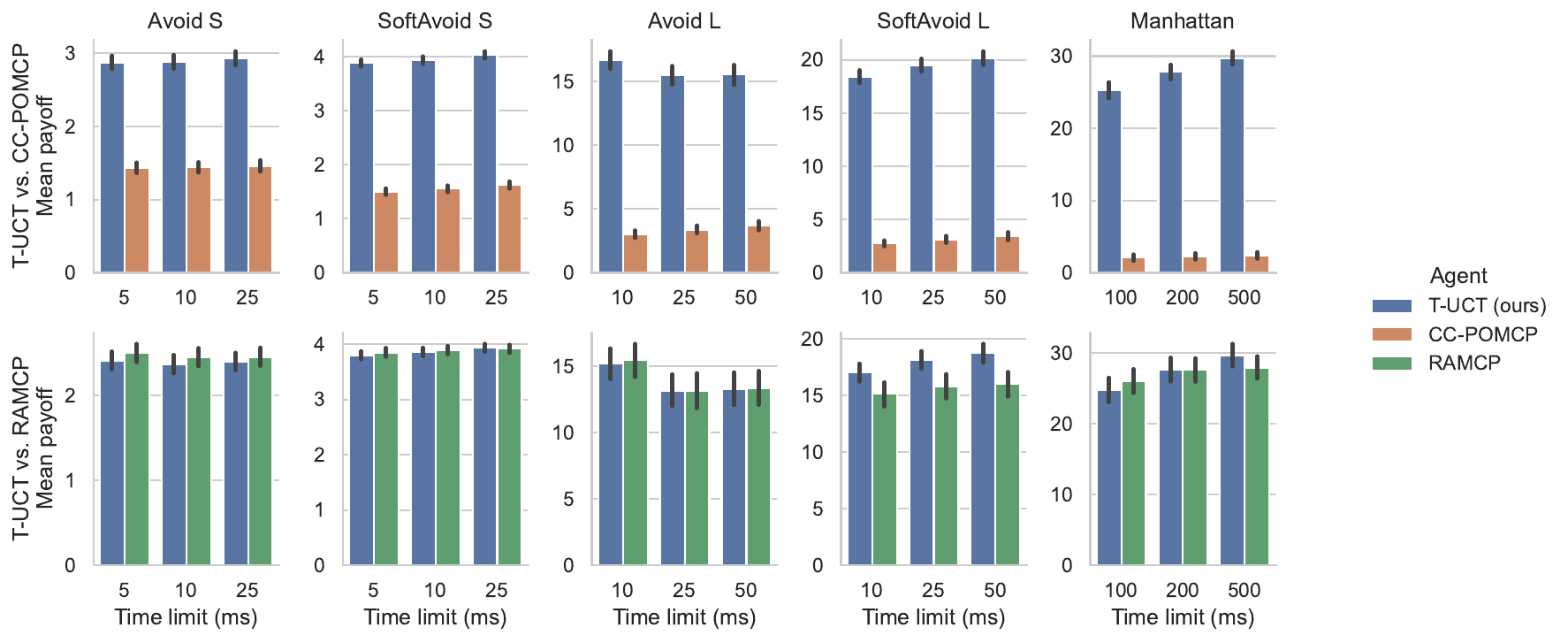}
    \caption{
        Mean payoff of T-UCT compared to each baseline.
        The average is calculated only over instances satisfied (in the weak sense) by both of the considered algorithms.}
   \label{fig:comparison}
\end{figure*}

\paragraph{Results}

%
%

The safety results are shown in Figure~\ref{fig:results}. 
The first row shows the overall safety on individual environments. Each sub-plot displays the fraction of instances satisfied under the given time limits.
Both T-UCT and CC-POMCP consistently solve a substantial proportion of instances, irrespective of the environment,
while RAMCP struggles to find feasible solutions, especially in the \emph{Manhattan} environment.

The second row of Figure \ref{fig:results} contains a breakdown of solved instances across different thresholds.
All algorithms have similar satisfactions ratios when no cost is permitted (threshold 0);
however, RAMCP quickly deteriorates with increasing threshold as it is not able to accurately estimate the expected cost 
(only its positivity). T-UCT and CC-POMCP keep the number of satisfied instances, or even improve it, with more risk permitted.

Figure \ref{fig:comparison} compares the \emph{payoffs} collected by T-UCT to other two algorithms,
and only considers the instances satisfied by both of the compared algorithms.
From the first row, it is clear that CC-POMCP is overly pessimistic in its actions, resulting in overconservative behavior.
On the other hand, due to its reward-oriented exploration, RAMCP often finds valuable strategies once it finds feasible ones.
Nevertheless, T-UCT is still able to achieve similar or higher payoffs than RAMCP on all environments.
%
%
%
%

At the same time, T-UCT satisfied many more instances than RAMCP.
The difference is especially notable on the large environments with non-zero cost thresholds where RAMCP
satisifes at most half of the instances while T-UCT steadily achieves between $0.8$ and $1.0$ satisfaction rate.
Regarding their performance in terms of payoffs, there is little statistically significant difference, although T-UCT achieved notably higher payoffs on large \emph{SoftAvoid} instances.

On the small benchmarks, T-UCT is able to identify nearly optimal trade-offs between cost and reward.
On these environments, CC-POMCP both solved fewer instances and achieved lower payoffs than T-UCT.
With larger environments, CC-POMCP's safety performance is on par with T-UCT or better. However, as we can see from the payoff values,
this is mainly caused by the fact that CC-POMCP played extremely conservatively.
The difference is most prominent in the \emph{Manhattan} environment,
where we observe up to a tenfold gap in payoff between CC-POMCP and T-UCT, despite CC-POMCP achieving superior $\sat$ values.
This is caused by CC-POMCP refusing almost all of the maintenance requests, thus failing to find a valuable policy.

Notably, T-UCT was the only algorithm capable of solving a substantial number of \emph{Manhattan} instances (approx. 74\% $\sat_M$, 100\% $\sat_W$) while achieving high payoffs.
As all the algorithms performed at most approx. 500 MCTS samples per step (an order of magnitude less than in other benchmarks, see Table~\ref{tbl:simulations}), this demonstrates T-UCT's data efficiency.

In summary, while CC-POMCP manages to find safe trajectories, it does so at the expense of the collected
rewards. On the other hand, RAMCP generally accumulated a large amount of
reward, although frequently disregarding the safety of its behavior in the process.

\begin{table}[ht!]
    \begin{tabular}{lrrrr}
        \toprule
        Environment & $t$ & \rotatebox{25}{T-UCT (ours)} & \!\!\!\!\!\!\!\!\rotatebox{25}{CC-POMCP} & \!\!\!\!\!\!\!\!\rotatebox{25}{RAMCP} \\
        \midrule
        Avoid S & 10 & 324 & 954 & 2580 \\
        SoftAvoid S & 10 & 268 & 874 & 1920 \\
        Avoid L & 50 & 1017 & 2634 & 1378 \\
        SoftAvoid L & 50 & 1041 & 2587 & 1458 \\
        Manhattan & 500 & 387 & 349 & 427 \\
        \bottomrule
        \end{tabular}
    \caption{
        Average number of simulations per step performed by each algorithm.
        The low number of simulations by T-UCT is due to its higher computational cost during backpropagation.
        The drop in RAMCP's simulations for larger environments is caused by its two-phase nature.
    }
    \label{tbl:simulations}
\end{table}


   



\paragraph{Discussion}
Per our observation, the suboptimal performance of CC-POMCP is caused by two factors:
the unsound update rule, which
occasionally sets \( \thr \) to a value that is not achievable; and the slow convergence of the Lagrange multiplier~$\lambda$. For the latter reason, the algorithm often achieves low payoffs, even if the computed policies are safe.

RAMCP is often able to find relatively valuable strategies once it finds feasible ones, the latter task being its weak point. The problem of greedy exploration is pronounced
in the \emph{SoftAvoid} task, where rewards and costs are not directly correlated, and thus the sampled tree does not provide a good approximation of the original CMDP.
In \emph{Manhattan,} where it is easy to underestimate the costs, the cost budget of RAMCP often explodes as described at the end of Section \ref{sec:tuct},
which essentially leads to ignoring the constraint.

In summary, T-UCT strikes a good balance between both safety and payoff, a
property not shared by either of the baseline algorithms.
Although T-UCT performs significantly fewer simulations than the other algorithms,
it provides stable performance across all environments, thus demonstrating its sample efficiency.
Using Pareto curves, T-UCT stores the collected information in a well structured manner,
allowing it to effectively reason about the cost-reward trade-offs even during exploration and thus find safe and valuable policies.

\section{Conclusion and Future Work} \label{sec:conclu}

We introduced Threshold UTC, a MCTS-based planner for constrained MDPs utilizing online estimates of reward-cost trade-offs in search for safe and valuable policies. We presented an experimental evaluation of our approach, showing that it outperforms competing cost-constrained tree-search algorithms. Our aim for future work is to augment T-UCT with function approximators of Pareto sets, thus obtaining a general MCTS-based \emph{reinforcement learning} algorithm for cost-constrained  optimization.  

\section*{Acknowledgements}

This research was supported by the Czech Science Foundation grant no. GA23-06963S.

\bibliography{main-pareto-uct}

\newpage
\onecolumn
\appendix
\section{Appendix}
\subsection{Necessity to Consider the Action Outcome}
\label{app:ccpomcp-counterexample}
CMDP A in Figure~\ref{fig:cc-pomcp-mdp} with $\thr=0.5$ and $\gamma_c=1$ witnesses the unsoundness of CC-POMCP.
The update rule of CC-POMCP does not take the final state into account as it essentially updates the threshold to the value $\thract$ in T-UCT algorithm plus
adjusts for the immediate cost and the discount factor $\discount_\pen$.
Especially, if CC-POMCP deterministically selects an action $a$,
the threshold update only considers the immediate cost and the discount factor $\discount_\pen$, i.e.,
$$\thr' \gets \frac{\thr - \imcost(h, a)}{\discount_\pen}.$$
In the case of CMPD A, the agent keeps the threshold at $0.5$
regardless of whether the final state of action $a_1$ is $s_2$ or $s_3$. From state $s_2$ onwards,
the agent incurs an expected cost and reward of $0.5$. However, from state $s_3$,
it always incurs a cost of $1$. Therefore, the total expected cost of CC-POMCP from $s_0$ is $0.75$, which violates the constraint.

\begin{figure}[t]
    \centering
    \begin{tikzpicture}[scale=0.7]
        \node[draw, circle] (s0) at (0,1) {$s_0$};
        \node[draw] (a1) at (0,0) {$a_1$};
        \node[draw, circle] (s2) at (-2,-2) {$s_2$};
        \node[draw, circle] (s3) at (1,-2) {$s_3$};
        \node[draw] (a4) at (-3,-3) {$a_4$};
        \node[draw] (a5) at (-1,-3) {$a_5$};
        \node[draw] (a6) at (1,-3) {$a_6$};
        \node[draw, circle] (s7) at (-3,-5) {$s_7$};
        \node[draw, circle] (s8) at (-1,-5) {$s_8$};
        \node[draw, circle] (s9) at (1,-5) {$s_9$};

        \draw[->] (s0) -- (a1);
        \draw[<-, dashed] (s2) -- node[left] {$\substack{p=0.5}$}  (a1);
        \draw[->, dashed] (a1) -- node[right] {$\substack{p=0.5}$} (s3);
        \draw[->] (s2) -- (a4);
        \draw[->, dashed] (a4) -- node[left] {$\substack{\rew=1\\ \pen = 1}$} (s7);
        \draw[->] (s2) -- (a5);
        \draw[->, dashed] (a5) -- node[left] {$\substack{\rew=0\\ \pen = 0}$} (s8);
        \draw[->] (s3) -- (a6);
        \draw[->, dashed] (a6) -- node[left] {$\substack{\rew=0\\ \pen = 1}$} (s9);
    \end{tikzpicture}
    \caption{CMDP A}
    \label{fig:cc-pomcp-mdp}
\end{figure}

\newcommand{\C}{\mathcal{C}}
\newcommand{\poltuct}{\pi_{\text{T-UCT}}}
\newcommand{\etrans}[1]{\hat{\trans}^{#1}}

\subsection{Proofs}
\label{app:proofs}
In what follows, we fix a CMDP $\C$ with dynamics $\trans$, a length of the horizon $T$ we optimize over, and the initial threshold $\thr$.
We further use $\sigma^0, \ldots, \sigma^{T-1}$ to denote the random sequence of single-step policies computed by T-UCT on line \ref{aline:main-act-select} of Algorithm~\ref{alg:pareto-uct},
and $\poltuct$ to denote the policy which is the concatenation of the single-step policies.
The sequence $\thr^0, \ldots, \thr^{T-1}$ denotes the random thresholds computed by T-UCT with $\thr^0=\thr$, and $\hat{\trans}^0, \ldots, \hat{\trans}^{T-1}$ denotes the random sequence of estimates of $\trans$ before individual steps.
Additionally, $H$ denotes the set of all histories of length at most $T$ reachable from the initial state with a positive probability.
Finally, we use $\expect_{\trans'}$ to denote the expected value under particular dynamics $\trans'$.

We define a threshold $\thr'$ to be $\hat{\trans}^\tstep$-feasible from a history $h$ if it is feasible according to the Pareto set estimates of T-UCT at time $t$.
Since the Bellman equations \eqref{eq:bellman-h} and \eqref{eq:bellman-ha} compute precise Pareto sets if the environment follows the dynamics $\hat{\trans}^\tstep$,
the threshold $\thr'$ is $\hat{\trans}^\tstep$-feasible iff there exists a policy that yields $\acost$ at most $\thr'$ under the dynamics $\hat{\trans}^\tstep$.

We begin by proving the property \eqref{eq:thr-update-exp-decrease}:
\begin{lemma}
\label{lem:thr-update-exp-decrease}
Let $h \in H$ be a history of length $\tstep$. The following holds:
\begin{equation*}
    \thr^\tstep \geq \expect_{\etrans{\tstep}}\left[\pen(h, a, s) + \discount_\pen \cdot \thr^{\tstep+1}\mid \sigma^\tstep\right].
\end{equation*}
\end{lemma}

\begin{proof}
    For brevity, we use $\sigma := \sigma^\tstep$, and $\etrans{} := \etrans{\tstep}$.
    We first show that $\expect[\thract \mid \sigma] = \thr^\tstep$. If $\sigma$ is a deterministic policy,
    then T-UCT sets $\thract = \thr$ by the definition of $\thract$, and the claim follows.
    Otherwise, $\thract$ is equal to $c_l$ or $c_h$ with probability $\sigma_l$ or $\sigma_h$,
    respectively, where the probabilities are computed precisely to satisfy $\thr^\tstep = \sigma_l c_l + \sigma_h c_h$.
    It follows that $\expect[\thract \mid \sigma] = \thr^\tstep$.

    Therefore, it is enough to show that for every $a$ in the support of $\sigma^\tstep$, we have
    \begin{equation}
        \label{eq:act-thr-ineq}
        \thract \geq \expect_{\etrans{}}\left[\pen(h, a, s) + \discount_\pen \cdot \thr^{\tstep+1}\mid a \right].
    \end{equation}
In the ``mixing'' case, the new threshold $\thr^{\tstep+1}$ is equal to
one of $c_s$ from equation \eqref{eq:conv-comb}, which are chosen so as to satisfy \eqref{eq:act-thr-ineq}.

In the ``surplus'' case, $\sigma$ is the deterministic distribution corresponding to the choice of action $a$,
$\thract$ is equal to $\thr$, and $\thr^{\tstep+1}$ is set to
\begin{align}
    c_s + \left(\thract - c_{\max}\right)\frac{B - c_s}{\imcost(h, a) + \gamma_\pen B - c_{\max}}.
\end{align}
Recall the definition $\imcost(h, a) = \sum_{s} \hat{\trans}(s \mid h, a)\cdot \pen(h, a, s)$,
and note that by the definition of $c_s$ right above equation \eqref{eq:update-surplus}, the expectation of $c_s$ under $\hat{\trans}(s\mid h,a)$ is $$\expect_{\etrans{}}[c_s] = \frac{c_{\max} - \imcost(h, a)}{\gamma_\pen}.$$
It follows that
\begin{align*}
    &\expect_{\etrans{}}\left[ \pen(h,a,s) + \gamma_\pen\thr^{\tstep+1} \mid a\right] = \\
    &\imcost(h, a) + \gamma_\pen \expect_{\etrans{}}\left[c_s + \frac{\left(\thract - c_{\max}\right)\left(B - c_s\right)}{\imcost(h, a) + \gamma_\pen B - c_{\max}}\right] = \\
    &c_{\max} + \gamma_\pen \left(\thract - c_{\max}\right)\frac{B - \expect_{\etrans{}}\left[c_s\right]}{\imcost(h, a) + \gamma_\pen B - c_{\max}} = \\
    &c_{\max} + \left(\thract - c_{\max}\right)\frac{\gamma_c B - \left(c_{\max} - \imcost(h, a)\right)}{\imcost(h, a) + \gamma_\pen B - c_{\max}} = \thract.
\end{align*}

Finally, in the ``unfeasible'' case, we have
$$\thr^{\tstep+1} = c_s - \frac{c_{\min} - \thract}{\hat{\trans}(s \mid h, a) \gamma_\pen},$$
and by definition of $c_s$ below equation \eqref{eq:ramcp-update}, we have
$$\expect_{\etrans{}}[c_s \mid a] = \frac{c_{\min} - \imcost(h, a)}{\gamma_\pen}.$$
Hence
\begin{align*}
    &\expect_{\etrans{}}\left[ \pen(h,a,s) + \gamma_\pen\thr^{\tstep+1} \mid a\right] = \imcost(h, a) + \expect_{\etrans{}}\left[\gamma_\pen c_s - \frac{c_{\min} - \thract}{\hat{\trans}(s \mid h, a)}\right] = c_{\min} - k(c_{\min} - \thract) \leq \thract,
\end{align*}
where $k := \left|supp(\hat{\trans}(s \mid h, a))\right| \geq 1$ is the number of states with non-zero probability from $h$ under action $a$.
The inequality follows from the fact that $\thract < c_{\min}$.
\end{proof}

By further analysis of the formulas \eqref{eq:conv-comb} and \eqref{eq:update-surplus}, we can prove the property \eqref{eq:thr-update-feasible}:
\begin{lemma}
    \label{lem:thr-update-feasible}
    Let $h\in H$ be a history of length $\tstep$, and $c_s$ be as in the description of \texttt{UpdateThr} in Algorithm~\ref{alg:pareto-uct}.
    The following holds for every outcome $s$ of playing action $a$ in $h$:
    If $\thr^\tstep$ is $\etrans{}$-feasible from $h$, then $\thr^{\tstep+1} \geq c_s$; otherwise, $\thr^{\tstep+1} < c_s$.
\end{lemma}
\begin{proof}
    If $\thr^\tstep$ is $\etrans{}$-feasible, then we are in one of the cases ``mixing'' or ``surplus''.
    In the former case, $\thr^{\tstep+1}$ is directly equal to $c_s$.
    In the latter case, $\thr^{\tstep+1}$ is equal to
    \begin{align*}
        c_s + \left(\thract - c_{\max}\right)\frac{B - c_s}{\imcost(h, a) + \gamma_\pen B - c_{\max}}.
    \end{align*}
    Note that all three terms in the second summand are positive, hence indeed $\thr^{\tstep+1} \geq c_s$.
    
    Finally, if $\thr^\tstep$ is $\etrans{}$\nobreakdash-unfeasible, then $\thr^{\tstep+1}$ is equal to
    $$
    c_s - \frac{c_{\min} - \thract}{\hat{\trans}(s \mid h, a) \gamma_\pen},
    $$
    where $c_{\min}$ is the minimal cost achievable from $h$ according to $\etrans{}$.
    Since $\thract$ is unfeasible, $\thract < c_{\min}$, thus the second term is positive.
\end{proof}

\newcommand{\am}{a_{\min}}
\newcommand{\dist}{d}

We now prove that the estimates $\hat{\trans}$ can get arbitrarily close to the real dynamics $\trans$ with high probability,
provided that T-UCT is given sufficiently many MCTS iterations per step.

\begin{lemma}
    \label{lem:dyn-close}
    For every $p \in [0, 1)$ and $\dist > 0$, there exists $N$ such that with probability at least $p$,
    the $L^1$-distance between the real dynamics $\trans$ and the estimated dynamics $\etrans{\tstep}$ is less than $\dist$ for every $t$,
    provided T-UCT is given $N$ MCTS-iterations per step.
\end{lemma}
\begin{proof}
    We first analyze a single MCTS phase and show that after sufficiently many samples,
    T-UCT visits each history in $H$ arbitrarily many times with probability close to one.
    It is enough to show that if T-UCT performed infinitely many MCTS iterations,
    it would visit each history $h \in H$ infinitely many times almost surely (a.s.), i.e., with probability one.
    The latter claim can be shown by a simple induction on the length of the history $|h|$.
    By the assumption, the root is visited infinitely many times.
    Moreover, since there are finitely many actions and the payoffs are bounded,
    the number of times an action $a$ is selected in a history $h$ is in $\Theta(\log N(h))$ due to the exploration bonus,
    where $N(h)$ is the number of times the history $h$ has been visited during the MCTS phase so far.
    Thus if $h$ is visited infinitely many times a.s., each action in $h$ is selected infinitely many times a.s.
    Therefore, also every outcome $has \in H$ is sampled infinitely many times a.s.
    We conclude that for any $p' \in [0, 1)$ and $M$, there exists $N$ such that with probability at least $p$, after $N$ iterations of the MCTS phase,
    each history $h \in H$ is visited at least $M$ times.

    Due to the law of large numbers, the estimate $\etrans{\tstep}$ of the transition function tends to the real transition function $\trans$ with high probability
    as the number of samples $M$ grows. Hence, given the claim proved in the previous paragraph,
    for every probability $p' \in [0, 1)$, there exists $N$ such that
    with probability at least $p'$, the estimates $\etrans{\tstep}$ are $\dist$-close (in $L^1$ metric) to $\trans$ after $N$ iterations of the MCTS phase.
    
    Finally, if $N$ is the number from previous paragraph with the choice $p':=1-\frac{1-p}{T}$, the union bound guarantees that with probaility at least $p$,
    all the sample estimate $\hat{\trans}^\tstep$ are $\dist$\nobreakdash-close to $\trans$.
\end{proof}

\newcommand{\oeps}{\mathcal{O}(\varepsilon)}

The last ingredient is showing that once the estimates $\hat{\trans}$ are close to the real dynamics $\trans$,
the policy computed by T-UCT is $\varepsilon$-close to feasible.

\begin{lemma}
    \label{lem:thr-achieved}
    Let $\varepsilon \in \reals_{>0}$  and $h \in H$ be a history of length $\tstep$.
    Further, assume the estimates $\etrans{\tstep}$ are at most $\dist$-close to $\trans$ in the $L^1$ metric.
    Then for sufficiently small value of $\dist$, the following holds $$\acost_{\poltuct}(h) \leq \max\{\thr^\tstep, \thr^\tstep_{\min}\} + \frac{T-\tstep}{T}\varepsilon,$$
    where $\thr^\tstep_{\min}$ is the minimal cost achievable from $h$ under $\trans$.
\end{lemma}
\begin{proof}
    We prove the claim by induction on the length of the history $\tstep$ in the decreasing order.
    The base case $\tstep = T$ is trivial since there is only one policy --- the empty policy.
    In that case $\acost_{\poltuct}(h) = 0 = \thr^\tstep_{\min}$.

    Now suppose the claim holds for all histories of length $\tstep + 1$.
    Let $h$ be a history of length $\tstep$ and $\thr^\tstep$ be a threshold.
    Recall that $\poltuct$ is the policy computed by T-UCT.
    
    By Lemma \ref{lem:thr-update-exp-decrease}, $\thr^{\tstep+1}$ satisfies
    \begin{equation*}
        \thr^\tstep \geq \expect_{\etrans{\tstep}}\left[\pen(h, a, s) + \discount_\pen \cdot \thr^{\tstep+1} \mid \poltuct\right].
    \end{equation*}
    For sufficiently small $\dist$, the inequality also holds for the real transition function $\trans$, up to a small error:
    \begin{equation}
        \label{eq:lem:thr-achieved-1}
        \thr^\tstep+\frac{1}{2T}\varepsilon \geq \expect_{\trans}\left[\pen(h, a, s) + \discount_\pen \cdot \thr^{\tstep+1} \mid \poltuct \right].
    \end{equation}
    By the induction hypothesis, it holds that
    $$\acost_{\poltuct}(has) \leq \max\{\thr^{\tstep+1}, \thr^{\tstep+1}_{\min}\} + \frac{T-\tstep-1}{T}\varepsilon.$$
    We distinguish two cases based on whether $\thr^\tstep$ is $\etrans{\tstep}$-feasible from $h$.
    
    If $\thr^\tstep$ is $\etrans{\tstep}$-feasible, then $\thr^{\tstep+1}$ is greater or equal to $c_s$ (by Lemma \ref{lem:thr-update-feasible}), which, for a sufficiently small $\dist$, is greater than $\thr^{\tstep+1}_{\min} - \frac{1}{2T}\varepsilon$ (since $c_s$ is $\etrans{t}$-feasible by its definition above equation \eqref{eq:update-surplus}).
    Therefore, we have
    \begin{equation}
        \label{eq:lem:thr-achieved-2}
        \acost_{\poltuct}(has) \leq \thr^{\tstep+1} + \frac{2T-2\tstep-1}{2T}\varepsilon.
    \end{equation}
    So by \eqref{eq:lem:thr-achieved-1} and \eqref{eq:lem:thr-achieved-2} we have
    \begin{align*}
        & \thr^\tstep+\frac{1}{2T}\varepsilon \geq \expect_{\trans}\left[\pen(h, a, s) + \discount_\pen \cdot \thr^{\tstep+1}\mid \poltuct \right] \geq \\
        & \expect_{\trans}\left[\pen(h, a, s) + \discount_\pen \cdot (\acost_{\poltuct}(has) - \frac{2T-2\tstep-1}{2T}\varepsilon)\mid \poltuct \right] \geq \\
        & \expect_{\trans}\left[\pen(h, a, s) + \discount_\pen \cdot \acost_{\poltuct}(has)\mid \poltuct \right] - \frac{2T-2\tstep-1}{2T}\varepsilon,
    \end{align*}
    which yields
    \begin{align*}
     & \acost_{\poltuct}(h) = \expect_{\trans}\left[\pen(h, a, s) + \discount_\pen \cdot \acost_{\poltuct}(has) \mid \poltuct \right] \leq \thr^\tstep + \frac{2T-2\tstep}{2T}\varepsilon.
    \end{align*}

    If $\thr^\tstep$ is $\etrans{\tstep}$\nobreakdash-unfeasible from $h$, then, for a sufficiently small value of $\dist$, it holds that
    \begin{equation}
        \label{eq:lem:thr-achieved-3}
        \thr^\tstep \leq \thr^\tstep_{\min} + \frac{1}{4T}\varepsilon.
    \end{equation}
    Further, by Lemma \ref{lem:thr-update-feasible}, we have $\thr^{\tstep+1} < c_s$.
    As $\thr^{\tstep}$ is $\etrans{\tstep}$\nobreakdash-unfeasible, $c_s$ is the minimal cost achievable from $has$ under $\etrans{\tstep}$ (by definition of $c_s$ below equation \eqref{eq:ramcp-update}).
    Therefore, for sufficiently small $\dist$, $c_s \leq \thr^{\tstep+1}_{\min} + \frac{1}{4T}\varepsilon$.
    Altogether, we obtain $\thr^{\tstep+1} < c_s \leq \thr^{\tstep+1}_{\min} + \frac{1}{4T}\varepsilon$.
    By combining the latter inequality with the induction hypothesis, we derive
    \begin{equation}
        \label{eq:lem:thr-achieved-4}
        \acost_{\poltuct}(has) \leq \thr^{\tstep+1}_{\min} + \frac{4T-4\tstep - 3}{4T}\varepsilon.
    \end{equation}
    Analogously to the previous case, by combinding \eqref{eq:lem:thr-achieved-3}, \eqref{eq:lem:thr-achieved-1}, and \eqref{eq:lem:thr-achieved-4}, we derive
    \begin{align*}
        \acost_{\poltuct}(h) \leq \thr^\tstep_{\min} + \frac{4T-4\tstep}{4T}\varepsilon
    \end{align*}

\end{proof}

Finally, we are ready to prove Theorem~\ref{thm:convergence}.
\begin{proof}
    Let $C = T \cdot \max_{s, a} c(s, a)$ be the maximal cost achievable in the CMDP $\C$, and let $p = 1- \frac{\varepsilon}{2C}$.
    Further let $\dist$ be such that $\eps$ in Lemma \ref{lem:thr-achieved} is smaller than $\frac{\varepsilon}{2}$.
    By Lemma \ref{lem:dyn-close}, there exists $N$ such that with probability at least $p$,
    all dynamics estimates $\hat{\trans}^\tstep$ are $\dist$-close to $\trans$ after $N$ MCTS iterations per step.
    The expected cost $\acost_{\poltuct}(s_0)$ is then at most
    $$p \cdot \left(\thr^0 + \frac{\varepsilon}{2}\right) + (1-p)C \leq \thr^0 + \frac{\varepsilon}{2} + \frac{\varepsilon}{2} = \thr + \varepsilon.$$
\end{proof}

\newpage
\subsection{Experiment Configurations}
\label{app:experiment-configs}
\begin{table}[t]
    \centering
    \renewcommand{\arraystretch}{1.5}
    \begin{tabular}{cccc}
    \toprule
    \textbf{Task} & \textbf{Dataset} & \textbf{Horizon} ($T$) & \textbf{Parameter Settings} \\
    \midrule
    \multirow{4}{*}{Avoid} & \texttt{GridworldSmall} & 100 & \begin{tabular}[c]{@{}c@{}} \hspace{0.2cm} $\thr \in \{0, 0.15, 0.35\}$ \\ $\trp \in \{0.2, 0.5\}$ \\ $\slp \in \{0, 0.2\}$ \hspace{0.2cm} \end{tabular} \\ \cline{2-4} 
     & \texttt{GridworldLarge} & 200 & \begin{tabular}[c]{@{}c@{}} \hspace{0.2cm} $\thr \in \{0, 0.15, 0.3\}$ \\ $\trp \in \{0.02\}$ \\ $\slp \in \{0, 0.2\}$ \hspace{0.2cm} \end{tabular} \\ \hline
    \multirow{4}{*}{SoftAvoid} & \texttt{GridworldSmall} & 100 & \begin{tabular}[c]{@{}c@{}} \hspace{0.2cm} $\thr \in \{0, 0.15, 0.3, 0.45, 0.6, 0.75\}$ \\ $\trp \in \{0.2\}$ \\ $\slp \in \{0, 0.2\}$ \hspace{0.2cm} \end{tabular} \\ \cline{2-4} 
     & \texttt{GridworldLarge} & 200 & \begin{tabular}[c]{@{}c@{}} \hspace{0.2cm} $\thr \in \{0, 0.15, 0.3\}$ \\ $\trp \in \{0.02\}$ \\ $\slp \in \{0, 0.2\}$ \hspace{0.2cm} \end{tabular} \\ \hline
     Manhattan & \texttt{Manhattan} & 200 & \begin{tabular}[c]{@{}c@{}} \hspace{0.2cm} $\thr \in \{0, 0.15, 0.3, 0.45, 0.6\}$ \\ $radius \in \{0.2, 0.4\}$ km \\ $period \in \{50, 100\}$ \\ $delay \in \{10, 20\}$ \hspace{0.2cm} \end{tabular} \\ \hline
    \bottomrule
    \end{tabular}
    \caption{Configuration used for individual tasks.}
\label{tab:evaluation_configs}
\end{table}

\paragraph{Exploration Constant}
MCTS-based methods are particularly sensitive to the selection of their exploration
constant $\explconst$. In our preliminary experiments, we tried values in the range $[0.1, 20]$ with $C=5$ yielding the best performance for RAMCP and CC-POMCP, and making little difference for T-UCT.
Hence, all the experiments are performed with $C=5$.

\paragraph{Gridworld}
The Gridworld maps were generated using our randomized map generator to create a diverse dataset, avoiding any bias towards specific topologies that might favor particular algorithms.
The generator script is available in our project repository under \texttt{generator.py}, which also includes detailed descriptions of its parameters. The parameters utilized for the dataset generation are documented in the dataset files \texttt{HW\_SMALL.txt} and \texttt{HW\_LARGE.txt}.
The maps are represented in text format, where \texttt{B} denotes the initial tile, \texttt{G} indicates the goal tile, \texttt{\#} represents a wall, and \texttt{T}~signifies a trap.

The evaluated configurations with varying dataset, horizon $T$, $\thr$, $\trp$, and $\slp$ are summarized in Table~\ref{tab:evaluation_configs}.




\paragraph{Manhattan}
The Manhattan environment is based on the AEV benchmark first presented
in \cite{manhattan-paper}, the implementation of which is included in the Python package FiMDPEnv. \cite{fimdp}

The agent deterministically moves between junctions in Manhattan,
while incurring stochastic time delay based on real vehicle travel data from Uber. \cite{ubermovement} 
We augment the original Manhattan environment with targets that request periodic maintenance every $period$ time units. 
Whenever enough time has elapsed for a target to reach the $period$, and the target's location is less than $radius$ kilometers away from the agent, the environment transitions to a dummy decision state. In this state, the agent can accept one of the available
maintenance orders, or potentially decline any new jobs. In our experiments, we investigate several values of the
$radius$ parameter, resulting in varying amounts of orders
available to the agent during execution.

Once a maintenance order has been accepted, the agent must make the trip to its location on the map. For successfully reaching the target, he receives 1 unit of reward, but if the drive takes longer than $delay$ time units, he also incurs 0.1 units of penalty. 
The agent thus needs to carefully consider which orders to accept, in order not to violate the specified penalty threshold. 

The evaluated configurations with varying horizon $T$, $\thr$, $radius$, $period$, and $delay$ are summarized in Table~\ref{tab:evaluation_configs}.

\paragraph{Time Limits}
For RAMCP, it is not possible to simply limit the time budget for the MCTS phase since the LP phase can take non-negligible time.
It is, however, difficult to adjust the time limits by a simple rule, since the time spent in the LP phase depends on the problem topology and the size of the sampled search tree.
For this reason, we hand crafted the RAMCP time limits based on empirical observations so that its real time per step is comparable to times of T-UCT and CC-POMCP. See function \texttt{ramcp\_time\_correction} in \texttt{eval.py} for details.
See  Section \ref{app:additional-data} to see the overview of the time limits and the real time per step for each algorithm.

\paragraph{Execution}
The experiments were run on 5 machines with AMD Ryzen 9 3900X, 32GB of RAM, with Ubuntu 22.04.4 LTS. We used OR-Tools \cite{ortools} version 9.10 for the LP solvers and FiMDPEnv version 1.0.4. to build the Manhattan benchmark.
The jobs were distributed across the machines using Ray 2.31.0.
The total evaluation time was approx. 53 hours for the Gridworld tasks and approx. 75 hours for the Manhattan task.

\newpage
\subsection{Additional Data}\label{app:additional-data}

\newcolumntype{H}{>{\setbox0=\hbox\bgroup}c<{\egroup}@{}}
\begin{table}[!hb]
    \centering
\renewcommand{\arraystretch}{1.12}

\begin{tabular}{lllrrrr}
    \toprule
    agent & env & time limit &   real time & $SAT_M$ & $SAT_W$ & samples \\
    \midrule
    \multirow{15}{*}{CC-POMCP} & \multirow{3}{*}{Avoid L} & 10 & 10.14 & 0.73 & 0.92 & 583.86 \\
     &  & 25 & 25.35 & 0.76 & 0.93 & 1370.62 \\
     &  & 50 & 50.70 & 0.77 & 0.94 & 2633.91 \\
    \cline{2-7}
     & \multirow{3}{*}{Avoid S} & 5 & 5.15 & 0.58 & 0.63 & 500.29 \\
     &  & 10 & 10.21 & 0.58 & 0.64 & 954.47 \\
     &  & 25 & 25.40 & 0.59 & 0.64 & 2298.28 \\
    \cline{2-7}
     & \multirow{3}{*}{Manhattan} & 100 & 102.58 & 0.93 & 1.00 & 78.34 \\
     &  & 200 & 203.59 & 0.95 & 1.00 & 145.02 \\
     &  & 500 & 506.31 & 0.94 & 1.00 & 349.42 \\
    \cline{2-7}
     & \multirow{3}{*}{SoftAvoid L} & 10 & 10.12 & 0.90 & 1.00 & 560.45 \\
     &  & 25 & 25.26 & 0.90 & 1.00 & 1330.35 \\
     &  & 50 & 50.48 & 0.91 & 1.00 & 2587.02 \\
    \cline{2-7}
     & \multirow{3}{*}{SoftAvoid S} & 5 & 5.15 & 0.92 & 0.97 & 450.89 \\
     &  & 10 & 10.22 & 0.92 & 0.97 & 874.16 \\
     &  & 25 & 25.54 & 0.92 & 0.97 & 2117.45 \\
    \cline{1-7} \cline{2-7}
    \multirow{15}{*}{RAMCP} & \multirow{3}{*}{Avoid L} & 10 & 15.15 & 0.34 & 0.58 & 496.79 \\
     &  & 25 & 35.38 & 0.40 & 0.58 & 879.47 \\
     &  & 50 & 74.38 & 0.41 & 0.59 & 1377.64 \\
    \cline{2-7}
     & \multirow{3}{*}{Avoid S} & 5 & 12.93 & 0.63 & 0.73 & 1207.82 \\
     &  & 10 & 25.56 & 0.65 & 0.75 & 2579.67 \\
     &  & 25 & 57.99 & 0.69 & 0.79 & 7005.59 \\
    \cline{2-7}
     & \multirow{3}{*}{Manhattan} & 100 & 118.63 & 0.23 & 0.42 & 102.40 \\
     &  & 200 & 231.98 & 0.26 & 0.44 & 189.56 \\
     &  & 500 & 569.76 & 0.25 & 0.45 & 426.91 \\
    \cline{2-7}
     & \multirow{3}{*}{SoftAvoid L} & 10 & 15.05 & 0.54 & 0.67 & 518.45 \\
     &  & 25 & 34.62 & 0.55 & 0.67 & 922.17 \\
     &  & 50 & 71.65 & 0.52 & 0.67 & 1457.94 \\
    \cline{2-7}
     & \multirow{3}{*}{SoftAvoid S} & 5 & 11.97 & 0.66 & 0.83 & 813.39 \\
     &  & 10 & 25.48 & 0.68 & 0.87 & 1920.34 \\
     &  & 25 & 63.31 & 0.70 & 0.90 & 5811.67 \\
    \cline{1-7} \cline{2-7}
    \multirow{15}{*}{T-UCT (ours)} & \multirow{3}{*}{Avoid L} & 10 & 10.21 & 0.49 & 0.78 & 283.23 \\
     &  & 25 & 25.39 & 0.69 & 0.91 & 587.01 \\
     &  & 50 & 50.85 & 0.76 & 0.94 & 1016.85 \\
    \cline{2-7}
     & \multirow{3}{*}{Avoid S} & 5 & 5.11 & 0.70 & 0.83 & 181.05 \\
     &  & 10 & 10.16 & 0.75 & 0.86 & 324.23 \\
     &  & 25 & 25.30 & 0.78 & 0.87 & 704.96 \\
    \cline{2-7}
     & \multirow{3}{*}{Manhattan} & 100 & 113.64 & 0.72 & 0.97 & 93.04 \\
     &  & 200 & 223.63 & 0.72 & 0.99 & 173.84 \\
     &  & 500 & 546.64 & 0.75 & 1.00 & 386.88 \\
    \cline{2-7}
     & \multirow{3}{*}{SoftAvoid L} & 10 & 10.21 & 0.86 & 1.00 & 278.04 \\
     &  & 25 & 25.38 & 0.83 & 1.00 & 593.84 \\
     &  & 50 & 50.81 & 0.81 & 1.00 & 1041.50 \\
    \cline{2-7}
     & \multirow{3}{*}{SoftAvoid S} & 5 & 5.11 & 0.97 & 1.00 & 150.63 \\
     &  & 10 & 10.16 & 0.96 & 1.00 & 268.06 \\
     &  & 25 & 25.31 & 0.96 & 0.99 & 573.71 \\
    \cline{1-7} \cline{2-7}
    \bottomrule
    \end{tabular}
    
    \caption{The detailed results of the experiments summarized in Figure 
\ref{fig:results}. The column \emph{samples} shows the mean number of MCTS samples used per decision.}

\end{table}




\end{document}